\documentclass[ejsv2,preprint]{imsart}

\RequirePackage[numbers]{natbib}
\RequirePackage[hidelinks]{hyperref}
\RequirePackage{amsmath}
\RequirePackage{amssymb,amsthm,mathtools}
\RequirePackage{microtype}

\startlocaldefs

\newtheorem{theorem}{Theorem}[section]
\newtheorem{proposition}[theorem]{Proposition}
\newtheorem{lemma}[theorem]{Lemma}
\newtheorem{corollary}[theorem]{Corollary}
\theoremstyle{definition}
\newtheorem{assumption}[theorem]{Assumption}
\newtheorem{definition}[theorem]{Definition}

\theoremstyle{remark}
\newtheorem{remark}[theorem]{Remark}

\newcommand{\eps}{\varepsilon}
\newcommand{\Pcal}{\mathcal P}
\newcommand{\E}{\mathbb E}
\newcommand{\Prob}{\mathbb P}
\newcommand{\osc}{\operatorname{osc}}

\newcommand{\dd}{\,\mathrm d}
\newcommand{\R}{\mathbb R}
\newcommand{\cX}{\mathcal X}
\newcommand{\cY}{\mathcal Y}

\endlocaldefs

\begin{document}

\begin{frontmatter}

\title{Uniform Statistical Convergence of Empirical Sinkhorn Potentials\\
with Exponential and Polynomial Dependence on the Regularization Parameter}
\runtitle{Uniform Statistical Convergence of Empirical Sinkhorn Potentials}

\begin{aug}
\author{\fnms{Denis} \snm{Belomestny}\ead[label=e1]{denis.belomestny@uni-due.de}}
\address{
Department of Mathematics,
University of Duisburg-Essen\\
\printead{e1}
}
\runauthor{D. Belomestny}
\end{aug}

\begin{abstract}
We study the empirical Sinkhorn estimator of the entropic optimal transport potentials under the uniform loss.
Since the potentials are only unique up to additive constants, we measure the error using the quotient supremum norm, defined as $d_\infty([u],[v]) = \inf_{a\in\R}\|u-v-a\|_\infty = \tfrac12\osc(u-v)$.
For a fixed regularization parameter $\eps>0$, we establish a non-asymptotic statistical rate of $n^{-1/2}$.
This is achieved by combining the Birkhoff--Hopf contraction theorem with entropy bounds on normalized kernel sections.
However, the constant in this bound grows exponentially with $1/\eps$.
To improve this, we isolate geometric conditions under which the empirical estimator maintains the $n^{-1/2}$ rate but features polynomial dependence on $1/\eps$.
The key requirement is a polynomial residual-stability estimate for the population Sinkhorn map.
We provide sufficient criteria for this, including a polynomial contraction property and a local inverse estimate.
Furthermore, we introduce two rigorously verifiable model classes—an $\eps$-weak residual-interaction class obtained after separable centering and another based on connected tight-edge graphs for fixed discrete costs—where the polynomial rate is guaranteed without relying on abstract resolvent assumptions.
Finally, we establish matching minimax lower bounds demonstrating that the $\eps n^{-1/2}$ rate cannot be uniformly improved in the bounded-interaction regime.
\end{abstract}



\end{frontmatter}

\section{Introduction}
\label{sec:introduction}

Let $(\cX,d_{\cX})$ and $(\cY,d_{\cY})$ be compact metric spaces, equipped with probability measures $\mu\in\Pcal(\cX)$ and $\nu\in\Pcal(\cY)$.
Given a continuous cost function $c:\cX\times\cY\to\R$ and a regularization parameter $\eps>0$, the entropic optimal transport problem seeks a coupling $\pi$ that minimizes the expected cost plus an entropy penalty $\eps\,\mathrm{KL}(\pi\mid\mu\otimes\nu)$.
Its dual potentials, $(f_\eps,g_\eps)$, uniquely solve the Schrödinger system up to additive constants: $f_\eps(x) = -\eps\log\int_{\cY} \exp((g_\eps(y)-c(x,y))/\eps)\nu(\dd y)$ and $g_\eps(y) = -\eps\log\int_{\cX} \exp((f_\eps(x)-c(x,y))/\eps)\mu(\dd x)$.
We define integral operators $(\mathsf T_\rho^{\cY}h)(x) := -\eps\log\int_{\cY} \exp((h(y)-c(x,y))/\eps)\rho(\dd y)$ and $(\mathsf T_\sigma^{\cX}u)(y) := -\eps\log\int_{\cX} \exp((u(x)-c(x,y))/\eps)\sigma(\dd x)$.
This allows us to express the potentials as $f_\eps=\mathsf T_\nu^{\cY}g_\eps$ and $g_\eps=\mathsf T_\mu^{\cX}f_\eps$.
Modulo constants, $f_\eps$ is a fixed point of the composition $\Phi_\eps := \mathsf T_\nu^{\cY}\mathsf T_\mu^{\cX}$.
Given independent and identically distributed samples $X_1,\ldots,X_n\sim\mu$ and $Y_1,\ldots,Y_n\sim\nu$, we construct the empirical measures $\widehat\mu_n=\frac1n\sum_{i=1}^n\delta_{X_i}$ and $\widehat\nu_n=\frac1n\sum_{j=1}^n\delta_{Y_j}$.
The empirical Sinkhorn potentials $(\widehat f_{\eps,n},\widehat g_{\eps,n})$ solve the empirical analogs of the Schrödinger system, meaning $\widehat f_{\eps,n}$ is the fixed point of $\widehat\Phi_{\eps,n} := \mathsf T_{\widehat\nu_n}^{\cY}\mathsf T_{\widehat\mu_n}^{\cX}$.
Throughout, the empirical Sinkhorn estimator denotes the exact solution of the empirical Schrödinger system. We do not include the additional optimization error arising from terminating Sinkhorn's algorithm after finitely many iterations.

We highlight two main statistical regimes for these empirical potentials.
First, assuming a bounded Lipschitz cost and finite entropy integrals for the normalized kernel-section classes, the potentials converge at the parametric rate $n^{-1/2}$, although the leading constant may scale exponentially with $1/\eps$.
Second, if the population Sinkhorn map additionally exhibits polynomial residual stability, the $n^{-1/2}$ rate holds with a constant that grows only polynomially in $1/\eps$.
To complement our upper bounds, we provide a matching minimax lower bound in the bounded-interaction regime, confirming that the parametric rate and the linear dependence on the regularization parameter cannot be uniformly improved.
Throughout this paper, $\delta \in (0,1)$ denotes a generic probability tolerance.

\paragraph{Contributions.}
The main contributions of this paper are the following.

\begin{enumerate}

\item
We prove a nonasymptotic parametric bound for empirical Sinkhorn potentials
in quotient sup norm. For every fixed $\eps>0$, the empirical
potential satisfies, with high probability,
\[
d_\infty([\widehat f_{\eps,n}],[f_\eps])
 \leq
 \frac{C_\eps}{\sqrt n}
 \bigl(1+\sqrt{\log(1/\delta)}\bigr),
\]
where $C_\eps$ is given explicitly. The proof separates empirical
fluctuation of the Sinkhorn operator from deterministic inverse stability of
its population fixed point.

\item
We formulate an abstract polynomial-stability principle. If the summed
entropy integral is bounded by
$A_0\eps^{-p}\sqrt{H_\eps}$, the envelope is bounded by
$E_0\eps^{-a}$, and the population residual satisfies an inverse
stability estimate of order $K_0\eps^{-r}$, then
\[
d_\infty([\widehat f_{\eps,n}],[f_\eps])
\le
\frac{CK_0}{\sqrt n}
\left[
A_0\eps^{1-r-p}\sqrt{H_\eps}
+
E_0\eps^{1-r-a}\sqrt{\log(4/\delta)}
\right].
\]
This identifies separately the contributions of empirical complexity,
kernel envelopes, and inverse stability.

\item
We give separate verifiable mechanisms for controlling the empirical
Gibbs-score classes. Lower-mass conditions on metric balls, combined with
local Hölder or Laplace-type behavior, yield polynomial envelope bounds for
normalized Gibbs kernels. Covering-number bounds are then obtained from
metric-entropy assumptions on the endpoint supports and Lipschitz
parameterizations of the kernel sections.

\item
We show that the polynomial stability assumptions are non-vacuous on two
proper classes of models. The first is a weak-interaction class
\[
c_\eps(x,y)
 =a_\eps(x)+b_\eps(y)
   +\eps\ell_\eps(x,y),
\]
with controlled oscillation and regularity of $\ell_\eps$.
The second is a finite-state class characterized by a connected tight graph
and a strict-complementarity condition. In both cases the relevant inverse
stability constants are controlled polynomially in $\eps^{-1}$.

\item
For the bounded weak-interaction class
$(\gamma_{\mathrm{osc}}=q=0)$, we obtain the concrete rate
\[
d_\infty([\widehat f_{\eps,n}],[f_\eps])
 \leq
 C\,\eps
 \frac{1+\sqrt{\log(1/\delta)}}{\sqrt n}.
\]
We complement this result with a two-point minimax lower bound of order
$\eps/\sqrt n$, including the corresponding deviation dependence.
Thus the joint dependence on $n$ and $\eps$ is optimal on this
model class up to universal constants.

\end{enumerate}

\section{Literature overview}
\label{sec:literature}

\paragraph{Statistical estimation of entropic optimal transport.}
A substantial literature studies the empirical approximation of entropic
optimal transport costs and Sinkhorn divergences. Early nonasymptotic results
established parametric $n^{-1/2}$ convergence for fixed regularization,
although with constants that may deteriorate rapidly as
$\eps\downarrow0$; see, for example,
\cite{GenevayEtAl2019,MenaNilesWeed2019}.
Dimension-free parametric rates for several quantities associated with
entropic optimal transport were subsequently obtained in
\cite{RigolletStromme2022}. Intrinsic- and lower-complexity adaptation were
investigated in \cite{Stromme2023,GroppeHundrieser2024}, where the statistical
complexity can depend on the smaller effective dimension of the two
marginals. These works primarily concern transport costs, couplings,
barycentric projections, or regression-type functionals, rather than
nonasymptotic estimation of the dual potentials in quotient sup norm.
We emphasize that these statistical bounds for the exact empirical potentials should be distinguished from algorithmic convergence rates of finite Sinkhorn iterations.

\paragraph{Stability and regularity of Schrödinger potentials.}
Qualitative stability of Schrödinger potentials and entropic couplings under
perturbations of the marginals was established in
\cite{GhosalNutzBernton2022,NutzWiesel2022sinkhorn}.
Quantitative stability estimates for regularized optimal transport were
developed in \cite{EcksteinNutz2022}, while smoothness properties of the
Schrödinger map, including Lipschitz estimates from Wasserstein spaces into
spaces of smooth functions, were studied in
\cite{CarlierChizatLaborde2024}.
These results provide important continuity and regularity mechanisms.
However, they do not by themselves yield a nonasymptotic
$n^{-1/2}$ bound for empirical potentials in quotient sup norm with
explicit polynomial dependence on the regularization parameter.

\paragraph{Limit theory for empirical potentials.}
For fixed $\eps>0$, functional delta-method arguments and
Hadamard differentiability have been used to derive central limit theorems
and bootstrap validity for empirical entropic potentials, transport maps,
and Sinkhorn divergences; see \cite{GoldfeldKatoRiouxSadhu2024}.
Gaussian-process limits for empirical Sinkhorn potentials and related
functionals were also obtained in
\cite{GonzalezSanzLoubesNilesWeed2022}.
This theory shows that a root-$n$ scale is natural when
$\eps$ is fixed. Its main purpose is asymptotic distribution
theory, whereas our objective is a finite-sample sup-norm estimate that
tracks explicitly the deterioration of the constants as
$\eps\downarrow0$.

\paragraph{Small regularization and Schrödinger bridges.}
The convergence of Schrödinger potentials to Kantorovich potentials as
$\eps\downarrow0$ was studied in
\cite{NutzWiesel2022potentials}, with more refined asymptotic results available in
semi-discrete and regular settings; see, for example,
\cite{AltschulerNilesWeedStromme2022}.
Statistical regimes in which the regularization parameter decreases with
sample size have recently received increasing attention
\cite{Mordant2024}.
Related plug-in estimators for Schrödinger bridges have been analyzed in
metrics on path measures \cite{PooladianNilesWeed2024}. Such bounds do not
directly imply uniform convergence of the endpoint potentials, because the
map from a coupling or path measure to its logarithmic scaling potentials
requires a separate quantitative inverse-stability argument.

\paragraph{Positioning of the present work.}
Very recently, Gonz\'alez-Sanz, Nutz and Stromme \cite{GonzalezSanzNutzStromme2026} obtained nonasymptotic parametric bounds for empirical regularized-OT potentials under an empirical $L^2$ loss, with explicit polynomial dependence on the regularization parameter. Their result does not provide uniform control of the potentials modulo constants. In contrast, our loss is the quotient supremum metric and our analysis is based on uniform empirical-process control combined with inverse stability of the Sinkhorn residual. Further related works include Puchkin et al.~\cite{PuchkinEtAl2025}, who studied the sample complexity of Schr\"odinger potential estimation under KL loss rather than the exact empirical Sinkhorn fixed point in quotient sup norm. Additionally, Greco and Tamanini \cite{GrecoTamanini2026} studied the stability of gradients and Hessians and the convergence of Sinkhorn iterates with polynomial regularization dependence, though without addressing the statistical sampling problem considered here.
The present paper isolates this inverse-stability step. We combine empirical control of the Sinkhorn residual with stability of the population Schr\"odinger map to obtain finite-sample estimates in the quotient metric
\[
d_\infty([f],[g])
  :=\inf_{a\in\mathbb R}\|f-g-a\|_\infty.
\]
For fixed $\eps$, the resulting rate is parametric in $n$.
Under additional, verifiable geometric assumptions, its dependence on
$\eps$ is polynomial. For an explicitly defined weak-interaction
class, we also prove a minimax lower bound matching the upper dependence on
$n$ and $\eps$.

\section{Parametric Rates with Exponential Dependence}
\label{sec:exponential}

\subsection{Deterministic Contraction and Empirical Processes}
Let $C_c := \sup c - \inf c$ be the oscillation of the cost function, and define the contraction factor $\tau_\eps := \tanh(C_c/(2\eps)) < 1$.
Furthermore, we define the normalized population sections:
\begin{align}
 r_x^\nu(y) &:= \frac{\exp((g_\eps(y)-c(x,y))/\eps)}{\int_{\cY}\exp((g_\eps(z)-c(x,z))/\eps)\nu(\dd z)}, \label{eq:exp-r}\\
 s_y^\mu(x) &:= \frac{\exp((f_\eps(x)-c(x,y))/\eps)}{\int_{\cX}\exp((f_\eps(z)-c(z,y))/\eps)\mu(\dd z)}. \label{eq:exp-s}
\end{align}
Let $\mathcal R_{\nu,\eps}=\{r_x^\nu:x\in\cX\}$ and $\mathcal R_{\mu,\eps}=\{s_y^\mu:y\in\cY\}$, and define the upper envelope bound $B_\eps=e^{2C_c/\eps}$.

\begin{lemma}
\label{lem:soft-basic}
For any probability measure $\rho$ and bounded functions $h,h'$, we have $\osc(\mathsf T_\rho^{\cY}h)\le C_c$ and $\osc(\mathsf T_\rho^{\cY}h-\mathsf T_\rho^{\cY}h')\le\osc(h-h')$.
The analogous statements hold for $\mathsf T_\sigma^{\cX}$.
\end{lemma}

\begin{lemma}
\label{lem:birkhoff}
For every probability measure $\rho$ on $\cY$, the strict contraction $\osc(\mathsf T_\rho^{\cY}h-\mathsf T_\rho^{\cY}h') \le \tau_\eps\osc(h-h')$ holds.
The same applies uniformly to $\mathsf T_\sigma^{\cX}$.
\end{lemma}

\begin{assumption}
\label{ass:metric-entropy}
The cost $c$ is Lipschitz continuous in both variables with constants $L_{\cX}$ and $L_{\cY}$.
Moreover, the entropy integrals $\Gamma_{\nu,\eps}$ and $\Gamma_{\mu,\eps}$ are finite, where $\Gamma_{\rho,\eps} := \int_0^{B_\eps} \sqrt{\log(2N(t,\mathcal R_{\rho,\eps},\|\cdot\|_\infty))}\dd t$ for $\rho\in\{\mu,\nu\}$. 
\end{assumption}

\begin{lemma}
\label{lem:exp-entropy}
Under Assumption~\ref{ass:metric-entropy}, the sections are bounded $B_\eps^{-1}\le r_x^\nu(y),s_y^\mu(x)\le B_\eps$ and are Lipschitz continuous: $\|r_x^\nu-r_{x'}^\nu\|_\infty \le 2B_\eps L_{\cX}\eps^{-1} d_{\cX}(x,x')$.
Analogous bounds hold for $s_y^\mu$.
\end{lemma}

\begin{lemma}
\label{lem:emp-process}
For a pointwise measurable class $\mathcal F$ bounded in $[0,B]$ with entropy integral $\Gamma(\mathcal F)$, and for any $\delta\in(0,1)$, independent observations $Z_1,\ldots,Z_n\sim P$ satisfy
\begin{equation}
 \sup_{h\in\mathcal F}|(P_n-P)h| \le \frac{C_{\mathrm D}}{\sqrt n} \left[ \Gamma(\mathcal F) + B\sqrt{\log\frac{2}{\delta}} \right]
\end{equation}
with probability at least $1-\delta$, for some universal constant $C_{\mathrm D}$.
\end{lemma}

\subsection{The Exponential Theorem}

\begin{theorem}
\label{thm:exp-main}
Suppose Assumption~\ref{ass:metric-entropy} holds.
For any $\delta\in(0,1)$, define the statistical error terms 
\begin{equation}
 q_{\nu,n}(\delta) := \frac{C_{\mathrm D}}{\sqrt n} \left[ \Gamma_{\nu,\eps} + B_\eps\sqrt{\log\frac4\delta} \right], \qquad
 q_{\mu,n}(\delta) := \frac{C_{\mathrm D}}{\sqrt n} \left[ \Gamma_{\mu,\eps} + B_\eps\sqrt{\log\frac4\delta} \right].
\end{equation}
With probability at least $1-\delta$, the estimation errors satisfy:
\begin{equation}
 d_\infty([\widehat f_{\eps,n}],[f_\eps]) \le \frac{\eps B_\eps}{1-\tau_\eps^2} \bigl(\tau_\eps q_{\mu,n}(\delta) + q_{\nu,n}(\delta)\bigr). \label{eq:exp-f-rate}
\end{equation}
The analogous bound holds for $d_\infty([\widehat g_{\eps,n}],[g_\eps])$ by swapping variables.
For a fixed $\eps>0$, this implies a rate of $O_{\Prob}(n^{-1/2})$.
\end{theorem}

\begin{corollary}
\label{cor:exp-gauge}
Fix a point $x_\circ\in\cX$ and define the zero-anchored functions $f_\eps^\circ=f_\eps-f_\eps(x_\circ)$ and $\widehat f_{\eps,n}^\circ = \widehat f_{\eps,n}-\widehat f_{\eps,n}(x_\circ)$.
Then twice the right-hand side of Equation~\eqref{eq:exp-f-rate} bounds $\|\widehat f_{\eps,n}^\circ-f_\eps^\circ\|_\infty$.
\end{corollary}

\begin{remark}
Because $B_\eps=e^{2C_c/\eps}$ and $(1-\tau_\eps^2)^{-1} = \cosh^2(C_c/(2\eps))$, Theorem~\ref{thm:exp-main} is purely a fixed-$\eps$ result.
The exponential behavior appears twice: from the worst-case lower bounds on normalizing integrals, and from global projective contractions.
\end{remark}

\section{Parametric Rates with Polynomial Dependence}
\label{sec:polynomial}

\subsection{Admissible Potentials and Uniform Score Classes}
Let $\mathcal U_\eps\subset C(\cX)/\R$ be a deterministic set that almost surely contains both $[f_\eps]$ and $[\widehat f_{\eps,n}]$.
For any $[u]\in\mathcal U_\eps$, let $g_u := \mathsf T_\mu^{\cX}u$.
We define normalized sections, invariant to the chosen representative:
\begin{align}
 s_{u,y}^{\mu}(x) &:= \frac{\exp((u(x)-c(x,y))/\eps)}{\int_{\cX}\exp((u(z)-c(z,y))/\eps)\mu(\dd z)}, \label{eq:poly-s}\\
 r_{u,x}^{\nu}(y) &:= \frac{\exp((g_u(y)-c(x,y))/\eps)}{\int_{\cY}\exp((g_u(z)-c(x,z))/\eps)\nu(\dd z)}. \label{eq:poly-r}
\end{align}
We gather these into classes $\mathcal S_{\mu,\eps} = \{s_{u,y}^{\mu}:[u]\in\mathcal U_\eps,\ y\in\cY\}$ and $\mathcal S_{\nu,\eps} = \{r_{u,x}^{\nu}:[u]\in\mathcal U_\eps,\ x\in\cX\}$.

\begin{assumption}
\label{ass:poly-complexity}
There exist constants $A_0,E_0<\infty$, exponents $p\ge a\ge0$, and a term $H_\eps\ge1$ (growing at most polynomially or polylogarithmically in $1/\eps$) such that the envelope $\sup_{h\in\mathcal S_{\mu,\eps}\cup\mathcal S_{\nu,\eps}} \|h\|_\infty \le E_0\eps^{-a}$ and the summed entropy integrals $\Gamma_{\mu,\eps}^{\mathrm{unif}} + \Gamma_{\nu,\eps}^{\mathrm{unif}} \le A_0\eps^{-p}\sqrt{H_\eps}$ are bounded.
\end{assumption}

\begin{assumption}
\label{ass:poly-stability}
There exist $K_0<\infty$ and $r\ge0$ such that for all $[u]\in\mathcal U_\eps$, the inequality $d_\infty([u],[f_\eps]) \le K_0\eps^{-r}d_\infty([u],[\Phi_\eps u])$ holds.
\end{assumption}

\begin{lemma}
\label{lem:uniform-map-perturbation}
Let $Z_{\mu,n}^{\mathrm{unif}} := \sup_{s\in\mathcal S_{\mu,\eps}} |(\widehat\mu_n-\mu)s|$ and $Z_{\nu,n}^{\mathrm{unif}} := \sup_{r\in\mathcal S_{\nu,\eps}} |(\widehat\nu_n-\nu)r|$.
On the event where $Z_{\mu,n}^{\mathrm{unif}}\vee Z_{\nu,n}^{\mathrm{unif}}\le 1/2$, the maps satisfy $\sup_{[u]\in\mathcal U_\eps} d_\infty([\widehat\Phi_{\eps,n}u],[\Phi_\eps u]) \le 2\eps( Z_{\mu,n}^{\mathrm{unif}} + Z_{\nu,n}^{\mathrm{unif}} )$.
\end{lemma}

\subsection{The Polynomial Theorem}

\begin{theorem}
\label{thm:poly-main}
Under Assumptions~\ref{ass:poly-complexity} and \ref{ass:poly-stability}, there is a universal constant $C$ such that for any $\delta\in(0,1)$, if the sample size $n$ satisfies:
\begin{equation}
 C\left[ A_0\eps^{-p}\sqrt{H_\eps} + E_0\eps^{-a}\sqrt{\log(4/\delta)} \right] n^{-1/2} \le 1/2, \label{eq:poly-smallness}
\end{equation}
then with probability at least $1-\delta$, we obtain the bound:
\begin{equation}
 d_\infty([\widehat f_{\eps,n}],[f_\eps]) \le \frac{C K_0}{\sqrt n} \left[ A_0\eps^{1-r-p}\sqrt{H_\eps} + E_0\eps^{1-r-a} \sqrt{\log\frac4\delta} \right].
\end{equation}
Provided $0<\eps\le1$, $p\ge a$, and $H_\eps\ge1$, this implies the simplified rate:
\begin{equation}
 d_\infty([\widehat f_{\eps,n}],[f_\eps]) \lesssim \eps^{1-r-p} \sqrt{\frac{H_\eps+\log(1/\delta)}{n}}. \label{eq:poly-summary}
\end{equation}
If $H_\eps \lesssim \eps^{-h}$, the simplified exponent becomes $\eps^{1-r-p-h/2}$. All displayed dependence on $\eps^{-1}$ is polynomial (ignoring logarithms inside $H_\eps$).
\end{theorem}

\subsection{Sufficient Conditions for Polynomial Stability}

\begin{proposition}
\label{prop:poly-contraction}
Suppose that, for some $\lambda_0>0$ and $r\ge0$ such that $0<\lambda_0\eps^r\le1$, and for all $[u],[v]\in\mathcal U_\eps$,
\begin{equation}
 d_\infty([\Phi_\eps u],[\Phi_\eps v]) \le (1-\lambda_0\eps^r)d_\infty([u],[v]).
\end{equation}
Then Assumption~\ref{ass:poly-stability} holds with $K_0=\lambda_0^{-1}$.
\end{proposition}

\begin{proposition}
\label{prop:local-inverse}
Let $\mathbb B_\cX=C(\cX)/\R$ with the quotient sup norm, and define $\mathcal G_\eps([u])=[u-\Phi_\eps u]$.
Assume $\mathcal G_\eps$ is Fr\'echet differentiable at $[f_\eps]$, its derivative $\mathcal L_\eps := D\mathcal G_\eps([f_\eps])$ is invertible, and satisfies $\|\mathcal L_\eps^{-1}\|_{\mathbb B_\cX\to\mathbb B_\cX} \le K_0\eps^{-r}$.
Further assume a quadratic remainder bound $\|\mathcal G_\eps([f_\eps+h])-\mathcal L_\eps[h]\|_{\mathbb B_\cX} \le M_0\eps^{-s}\|[h]\|_{\mathbb B_\cX}^2$ inside a local ball $\|[h]\|_{\mathbb B_\cX}\le\rho_\eps$.
Then, on the smaller ball defined by $\|[h]\|_{\mathbb B_\cX} \le \bar\rho_\eps := \min\{\rho_\eps, \frac{\eps^{r+s}}{2K_0M_0}\}$, we have the local stability bound $d_\infty([f_\eps+h],[f_\eps]) \le 2K_0\eps^{-r} d_\infty([f_\eps+h],[\Phi_\eps(f_\eps+h)])$.
\end{proposition}

\begin{corollary}
\label{cor:local-poly}
If the hypotheses of Proposition~\ref{prop:local-inverse} and Assumption~\ref{ass:poly-complexity} hold, and the empirical potential falls within the localization radius $d_\infty([\widehat f_{\eps,n}],[f_\eps]) \le \bar\rho_\eps$ with probability $1-\delta_{\mathrm{loc}}$, then the bound in Equation~\eqref{eq:poly-summary} holds with probability at least $1-\delta-\delta_{\mathrm{loc}}$, scaling $K_0$ by a factor of 2.
\end{corollary}

\subsection{Polynomial Kernel Envelopes}

\begin{lemma}
\label{lem:laplace-envelope}
Let $(Z,d)$ be compact and let $\rho\in\Pcal(Z)$ satisfy, for some $\mathfrak d>0$, $b_0,t_0>0$, $\rho(B(z,t))\ge b_0t^{\mathfrak d}$ for all $z\in\operatorname{supp}(\rho)$ and $0<t\le t_0$.
Let $\Psi_\eps$ be a family of continuous functions attaining their maxima at points $z_\eps\in\operatorname{supp}(\rho)$.
Assume that, for some $L,q>0$, $\Psi_\eps(z_\eps)-\Psi_\eps(z) \le Ld(z,z_\eps)^q$ whenever $d(z,z_\eps)\le t_0$.
If $0<\eps\le Lt_0^q$, then
\begin{equation}
 \sup_{z\in Z} \frac{e^{\Psi_\eps(z)/\eps}}{\int e^{\Psi_\eps(w)/\eps}\rho(\dd w)} \le \frac{eL^{\mathfrak d/q}}{b_0}\, \eps^{-\mathfrak d/q}. \label{eq:laplace-envelope}
\end{equation}
\end{lemma}

\begin{remark}
Lemma~\ref{lem:laplace-envelope} supplies polynomial envelopes for normalized Gibbs sections, but does not by itself establish entropy or residual stability.
\end{remark}

\section{Verifiable Model Classes}
\label{sec:verifiable-classes}

We introduce criteria relying strictly on the cost function and endpoint support geometry to secure polynomial bounds without unverified abstract curvature assumptions.

\subsection{Polynomial-Interaction Model Class}
\label{subsec:verifiable-class}

Separable terms in the cost do not alter the entropic coupling.
Consider a cost shifted by separable functions $a_\eps(x)$ and $b_\eps(y)$ leaving a residual $\bar c_\eps(x,y)$.
The error of the potentials exactly mirrors the error of the residual potentials.

\begin{definition}
\label{def:poly-interaction-class}
For $0<\eps\le1$, constants $A_0,L_0,D_{\cX},D_{\cY}<\infty$, $\gamma_{\mathrm{osc}},q\ge0$, and metric dimensions $\mathfrak d_{\cX},\mathfrak d_{\cY}\ge1$, let $\mathfrak M_\eps$ contain all models $(\cX,\cY,\mu,\nu,c_\eps)$ where the supports $\cX,\cY$ have covering numbers bounded by $(1+D_{\cX}/t)^{\mathfrak d_{\cX}}$ and $(1+D_{\cY}/t)^{\mathfrak d_{\cY}}$, respectively.
The cost takes the form $c_\eps(x,y) = a_\eps(x)+b_\eps(y)+\eps\ell_\eps(x,y)$.
The interaction $\ell_\eps$ must satisfy the oscillation bound $\osc_{\cX\times\cY}(\ell_\eps) \le A_0+\gamma_{\mathrm{osc}}\log(1/\eps)$ and feature Lipschitz conditions bounded globally as:
\begin{equation}
 \sup_y\operatorname{Lip}_x(\ell_\eps(\cdot,y)) \vee \sup_x\operatorname{Lip}_y(\ell_\eps(x,\cdot)) \le L_0\eps^{-q}.
\end{equation}
\end{definition}

\begin{remark}
This definition specifies an $\eps$-weak-interaction model where the nonseparable cost takes the form $c_\eps = a_\eps + b_\eps + \eps\ell_\eps$. It does not automatically establish polynomial dependence for an arbitrary fixed nonseparable continuous cost; genuinely fixed-cost configurations are handled instead by the discrete tight-graph theorem.
\end{remark}

\begin{lemma}
\label{lem:cross-centering}
For a continuous cost $c$, define its rectangular interaction oscillation by
\begin{equation}
 \Delta_\square(c) := \sup_{x,x'\in\cX, y,y'\in\cY} |c(x,y)+c(x',y')-c(x,y')-c(x',y)|.
\end{equation} 
Fix $x_0\in\cX$, $y_0\in\cY$, and set $a(x)=c(x,y_0)$, $b(y)=c(x_0,y)-c(x_0,y_0)$, and $\bar c(x,y)=c(x,y)-a(x)-b(y)$.
Then $\|\bar c\|_\infty\le\Delta_\square(c)$ and $\osc(\bar c)\le2\Delta_\square(c)$.
\end{lemma}

\begin{theorem}
\label{thm:verifiable-poly-class}
For a model in $\mathfrak M_\eps$, let $\mathfrak d_*= \max(\mathfrak d_{\cX},\mathfrak d_{\cY})$ and $D_*=\max(D_{\cX},D_{\cY})$.
Define the geometric complexity term:
\begin{equation}
 \Lambda_{\eps,\delta} := 1+ \sqrt{\mathfrak d_* \left\{ 1+\log(1+2D_*L_0\eps^{-q}) \right\}} +\sqrt{\log(4/\delta)}. \label{eq:Lambda-verifiable}
\end{equation}
There exists a universal numerical constant $C$ such that, for any $\delta\in(0,1)$, with probability at least $1-\delta$, we have
\begin{equation}
 d_\infty([\widehat f_{\eps,n}],[f_\eps]) \le C e^{5A_0} \eps^{\,1-5\gamma_{\mathrm{osc}}} \Lambda_{\eps,\delta} n^{-1/2}.
\end{equation} 
If both potentials are anchored at the same point $x^\circ$, the standard sup-norm error is at most double this bound.
\end{theorem}

\begin{corollary}
\label{cor:bounded-interaction}
If Definition~\ref{def:poly-interaction-class} holds with $\gamma_{\mathrm{osc}}=q=0$, then $d_\infty([\widehat f_{\eps,n}],[f_\eps]) \le C_{A_0,L_0,D_*,\mathfrak d_*} \eps n^{-1/2} (1+\sqrt{\log(1/\delta)})$ with probability at least $1-\delta$ for any $\delta\in(0,1)$.
\end{corollary}

\subsection{A Fixed-Cost Discrete Class with Strict Complementarity}
\label{subsec:discrete-class}

When dealing with a fixed and genuinely nonseparable cost, we can define a second verifiable class for finite endpoint spaces $\cX=\{1,\ldots,I\}$ and $\cY=\{1,\ldots,J\}$ with probabilities $\boldsymbol\mu=(\mu_1,\ldots,\mu_I)$ and $\boldsymbol\nu=(\nu_1,\ldots,\nu_J)$.

\begin{definition}
\label{def:tight-graph-model}
A finite model $(\boldsymbol\mu,\boldsymbol\nu,c)$ qualifies if all endpoint probabilities are strictly positive, and there exist vectors $a\in\R^I$, $b\in\R^J$ leaving a non-negative reduced cost $\bar c_{ij}:=c_{ij}-a_i-b_j\ge0$ whose zero set $E:=\{(i,j):\bar c_{ij}=0\}$ forms a connected bipartite graph.
We further require the existence of a coupling supported on $E$ that assigns strictly positive probability to every edge in $E$, and that off-edge costs are strictly positive.
\end{definition}

Let $\Pi_E(\boldsymbol\mu,\boldsymbol\nu)$ be the set of couplings with the prescribed marginals and support contained in $E$:
\begin{equation}
 \Pi_E(\boldsymbol\mu,\boldsymbol\nu) = \left\{ \pi\in\Pi(\boldsymbol\mu,\boldsymbol\nu) : \pi_{ij}=0 \text{ for } (i,j)\notin E \right\}.
\end{equation}
Under Definition~\ref{def:tight-graph-model}, this set is nonempty and compact. We define the entropy-minimizing coupling on this graph as:
\begin{equation}
 \pi^E := \operatorname*{argmin}_{\pi\in\Pi_E(\boldsymbol\mu,\boldsymbol\nu)} \mathrm{KL}(\pi\mid\boldsymbol\mu\otimes\boldsymbol\nu). \label{eq:edge-I-projection}
\end{equation}

\begin{lemma}
\label{lem:edge-jacobian}
Choose the gauge $v_J=0$.
There is a unique vector $\theta^E =(\alpha_1^E,\ldots,\alpha_I^E,v_1^E,\ldots,v_{J-1}^E)$ such that $\pi^E_{ij} = \mu_i\nu_j e^{\alpha_i^E+v_j^E}\mathbf 1_E(i,j)$.
The Jacobian of the marginal equations with respect to $\theta$ is invertible at $\theta^E$.
\end{lemma}

\begin{theorem}
\label{thm:discrete-poly}
For a model satisfying Definition~\ref{def:tight-graph-model}, there exist constants $\eps_0, \eta_0>0$ and $C_0<\infty$ (dependent only on the finite model) such that for $0<\eps\le\eps_0$ and any $\delta\in(0,1)$, if 
\begin{equation}
 t_{n,\delta} := \sqrt{\frac{\log(2(I+J)/\delta)}{2n}} \le\eta_0,
\end{equation}
then with probability at least $1-\delta$, we have $d_\infty([\widehat f_{\eps,n}],[f_\eps]) \le C_0\eps \sqrt{\log(2(I+J)/\delta)/n}$.
This fixed-cost model maintains the $n^{-1/2}$ rate and features a favorable linear scaling with $\eps$.
\end{theorem}

\section{A matching minimax lower bound in the bounded-interaction regime}
\label{sec:minimax-lower-bound}

We prove that the factors \(\eps\) and \(n^{-1/2}\) in Corollary~\ref{cor:bounded-interaction} cannot be improved uniformly over the bounded-interaction model class. In fact, the same conclusion holds for a two-point submodel. Throughout this section, an estimator is allowed to be an arbitrary measurable function of the observations; in particular, the result is not restricted to the empirical Sinkhorn estimator.

\subsection{A binary bounded-interaction submodel}

Fix \(A_0,L_0>0\), and set
\[
    \kappa:=\frac14\min\{A_0,L_0,1\}>0.
\]
Let
\[
    \cX=\cY=\{-1,+1\},
    \qquad
    d_{\cX}(x,x')=d_{\cY}(x,x')
      =\mathbf 1_{\{x\ne x'\}}.
\]
For \(t\in[-1/4,1/4]\), define
\[
    \mu(+1)=\mu(-1)=\frac12,
    \qquad
    \nu_t(+1)=\frac{1+t}{2},
    \qquad
    \nu_t(-1)=\frac{1-t}{2},
\]
and, for \(0<\eps\le 1\), let
\[
    c_\eps(x,y)
      :=-\eps\kappa xy
      =\eps\ell(x,y),
    \qquad
    \ell(x,y):=-\kappa xy.
\]
We write
\[
    \mathcal M_{\eps}^{\mathrm{bin}}(\kappa)
      :=
      \left\{
        (\cX,\cY,\mu,\nu_t,c_\eps):
        |t|\le\frac14
      \right\}.
\]

\begin{lemma}
\label{lem:binary-embedding}
The class
\(\mathcal M_{\eps}^{\mathrm{bin}}(\kappa)\) is a subclass of
the polynomial-interaction class of Definition~\ref{def:poly-interaction-class} with
\(\gamma_{\mathrm{osc}}=q=0\), \(D_{\cX}=D_{\cY}=1\), and
\(\mathfrak d_{\cX}=\mathfrak d_{\cY}=1\).
\end{lemma}

Let \(f_{\eps,t}\) denote the first Sinkhorn potential associated
with \((\mu,\nu_t,c_\eps)\), viewed as an element of
\(\R^{\cX}/\R\). The observation law is
\[
    \Prob_t^{(n)}
      :=\mu^{\otimes n}\otimes\nu_t^{\otimes n};
\]
the first sample does not contain information about \(t\), but it is
included so that the experiment agrees exactly with the empirical
two-marginal Sinkhorn experiment. Since \(\mu\) and \(c_\eps\) are common to every model in the subclass, they may even be revealed to the estimator; thus the lower bound holds for this more informative experiment.
For an estimator
\[
    \widehat f:
    \cX^n\times\cY^n
      \longrightarrow \R^{\cX}/\R,
\]
define
\[
    \mathcal R_{n,\eps}(\widehat f)
      :=
      \sup_{|t|\le 1/4}
      \E_t
      d_\infty\bigl([\widehat f],[f_{\eps,t}]\bigr).
\]

\subsection{Sensitivity of the binary Sinkhorn potential}

\begin{lemma}
\label{lem:binary-potential-separation}
There exist constants
\(\bar t_\kappa\in(0,1/4]\) and \(m_\kappa>0\), depending only on
\(\kappa\), such that, for every \(0<\eps\le1\) and
\(|t|\le\bar t_\kappa\),
\[
    d_\infty\bigl([f_{\eps,t}],
                  [f_{\eps,-t}]\bigr)
      \ge m_\kappa\eps |t|.
\]
One may take
\[
    m_\kappa
      =\frac{\tanh\kappa}{1-\tanh^2\kappa}.
\]
\end{lemma}

\subsection{The minimax lower bounds}

\begin{theorem}
\label{thm:matching-minimax-lower-bound}
Let \(A_0,L_0>0\), let \(\kappa\) be as above, and let
\(\mathcal M_{\eps}^{\mathrm{bin}}(\kappa)\) be the binary
bounded-interaction class. There exist constants
\(c_\kappa>0\) and \(n_\kappa\in\mathbb N\), depending only on
\(\kappa\), such that, for every \(0<\eps\le1\) and
\(n\ge n_\kappa\),
\[
    \inf_{\widehat f}
    \sup_{|t|\le1/4}
    \E_t
    d_\infty\bigl([\widehat f],[f_{\eps,t}]\bigr)
      \ge c_\kappa\frac{\eps}{\sqrt n}.
\]
Moreover, after possibly decreasing \(c_\kappa\),
\[
    \inf_{\widehat f}
    \sup_{|t|\le1/4}
    \Prob_t^{(n)}
    \left(
      d_\infty\bigl([\widehat f],[f_{\eps,t}]\bigr)
        \ge c_\kappa\frac{\eps}{\sqrt n}
    \right)
      \ge \frac38.
\]
The infima are over all measurable estimators based on the two
independent samples of size \(n\).
\end{theorem}

\begin{theorem}
\label{thm:matching-deviation-lower-bound}
There is a constant \(c_\kappa^{\mathrm{dev}}>0\), depending only on
\(\kappa\), such that the following holds. Let
\(0<\delta\le1/8\), and suppose
\[
    n\ge
    \frac{
      \bigl(1+\sqrt{\log(1/\delta)}\bigr)^2
    }{
      256\,\bar t_\kappa^2
    }.
\]
Then, for every \(0<\eps\le1\),
\[
    \inf_{\widehat f}
    \sup_{|t|\le1/4}
    \Prob_t^{(n)}
    \left[
      d_\infty\bigl([\widehat f],[f_{\eps,t}]\bigr)
      \ge
      c_\kappa^{\mathrm{dev}}\,
      \frac{\eps}{\sqrt n}
      \left(1+\sqrt{\log\frac1\delta}\right)
    \right]
    \ge\delta.
\]
One may take \(c_\kappa^{\mathrm{dev}}=m_\kappa/32\).
\end{theorem}

\begin{remark}
\label{rem:lower-bound-consequence}
Let \(\mathfrak M_\eps\) be any bounded-interaction class from
Definition~\ref{def:poly-interaction-class} with \(\gamma_{\mathrm{osc}}=q=0\) which contains
\(\mathcal M_{\eps}^{\mathrm{bin}}(\kappa)\).
Theorems~\ref{thm:matching-minimax-lower-bound} and
\ref{thm:matching-deviation-lower-bound} imply that no estimator can
improve uniformly on either the factor \(\eps\) or the factor
\(n^{-1/2}\) in Corollary~\ref{cor:bounded-interaction}. The deviation factor
\(1+\sqrt{\log(1/\delta)}\) is also unavoidable, up to constants, in
the stated range of \((n,\delta)\).
This lower bound concerns precisely the
\(\eps\)-weak-interaction regime
\(c_\eps=a_\eps+b_\eps+\eps\ell\).
It does not assert the same small-\(\eps\) behavior for an
arbitrary fixed nonseparable continuous cost.
\end{remark}

\section{Conclusion}
\label{sec:conclusion}
Empirical Sinkhorn potentials converge in the quotient sup norm at the parametric rate of $n^{-1/2}$ for a fixed $\eps>0$, assuming finite entropy of normalized kernel sections.
Standard global proofs naturally yield constants that scale exponentially with $1/\eps$.
Securing a polynomial dependence on $1/\eps$ without weakening the $n^{-1/2}$ rate relies on two primary conditions: the polynomial empirical complexity of the normalized Gibbs sections, and the polynomial residual stability of the population Sinkhorn map.
We provided practical frameworks to satisfy these conditions, including a verifiable weak-interaction class and a connected tight-graph discrete class.
Fully generalizing these uniform polynomial bounds to broader geometries remains tied to understanding the resolvent behavior of local diffusions, which may degenerate for highly disconnected domains.
In the bounded-interaction regime, the binary minimax lower bound shows that the rate $\eps n^{-1/2}$, including its sub-Gaussian confidence dependence, is optimal up to constants.

\begin{appendix}
\section{Proofs}
\label{sec:appendix-proofs}

\begin{proof}[Proof of Lemma~\ref{lem:soft-basic}]
To establish the oscillation bound, observe that for any $x,x'\in\cX$, the ratio of the integrands defining $\mathsf T_\rho^{\cY}h(x)$ and $\mathsf T_\rho^{\cY}h(x')$ is strictly bounded between $e^{-C_c/\eps}$ and $e^{C_c/\eps}$ due to the oscillation limit of $c$.
Taking logarithms immediately yields $\osc(\mathsf T_\rho^{\cY}h)\le C_c$.
To verify nonexpansiveness, define $a=\inf(h-h')$ and $b=\sup(h-h')$.
Pointwise, we have the bounds $e^{a/\eps}e^{(h'-c)/\eps} \le e^{(h-c)/\eps} \le e^{b/\eps}e^{(h'-c)/\eps}$.
Integrating these expressions with respect to $\rho$ and applying the $-\eps\log(\cdot)$ transformation reverses the inequalities, producing $-b\le \mathsf T_\rho^{\cY}h-\mathsf T_\rho^{\cY}h' \le-a$.
The oscillation of this difference is bounded by $b-a=\osc(h-h')$, which completes the proof.
\end{proof}

\begin{proof}[Proof of Lemma~\ref{lem:birkhoff}]
Define the exponentiated kernel $K_\eps(x,y)=e^{-c(x,y)/\eps}$ and the associated integral operator $(\mathsf L_\rho v)(x)=\int K_\eps(x,y)v(y)\rho(\dd y)$.
Utilizing Hilbert's projective metric for positive functions $d_{\mathrm H}(u,v) = \osc(\log(u/v))$, we bound the projective diameter of $\mathsf L_\rho$ by
\begin{equation*}
 \Delta(\mathsf L_\rho) \le \sup_{x,x',y,y'} \left| \log \frac{K_\eps(x,y)K_\eps(x',y')}{K_\eps(x,y')K_\eps(x',y)} \right| \le \frac{2C_c}{\eps}.
\end{equation*} 
Applying the standard Birkhoff--Hopf theorem \cite{Birkhoff1957} to this integral operator provides the contraction $d_{\mathrm H}(\mathsf L_\rho u,\mathsf L_\rho v) \le \tanh(\Delta(\mathsf L_\rho)/4) d_{\mathrm H}(u,v) \le \tau_\eps d_{\mathrm H}(u,v)$.
Applying this result with test functions $u=e^{h/\eps}$ and $v=e^{h'/\eps}$, and recognizing that scaling by $-\eps$ preserves the projective distance up to the factor $\eps$, yields the targeted result.
\end{proof}

\begin{proof}[Proof of Lemma~\ref{lem:exp-entropy}]
We first normalize the cost by adding a constant such that $0\le c\le C_c$.
From Lemma~\ref{lem:soft-basic}, the potential oscillations are bounded by $C_c$.
Normalizing the potentials to lie inside $[0,C_c]$ does not alter the ratios defining the sections, and the envelope limit $B_\eps^{-1}\le r_x^\nu(y),s_y^\mu(x)\le B_\eps$ naturally follows.
Next, let $a_x(y)=e^{(g_\eps(y)-c(x,y))/\eps}$.
We bound the logarithmic difference $|\log(a_x(y)/a_{x'}(y))| \le L_{\cX}\eps^{-1}d_{\cX}(x,x')$.
Because this same bound applies to the normalizer integral, we arrive at $|\log r_x^\nu(y)-\log r_{x'}^\nu(y)| \le 2L_{\cX}\eps^{-1}d_{\cX}(x,x')$.
The mean-value theorem and the upper envelope $B_\eps$ establish the required Lipschitz bounds.
The Lipschitz parameterizations give
\[
N(t,\mathcal R_{\nu,\varepsilon},\|\cdot\|_\infty)
\le
N\!\left(
\frac{\varepsilon t}{2B_\varepsilon L_{\mathcal X}},
\mathcal X,d_{\mathcal X}
\right),
\]
and analogously for \(\mathcal R_{\mu,\varepsilon}\).
Finiteness of the corresponding entropy integrals is imposed separately
in Assumption~\ref{ass:metric-entropy}.
\end{proof}

\begin{proof}[Proof of Lemma~\ref{lem:emp-process}]
Combining classical symmetrization with Dudley's entropy inequality (exploiting the domination of $L^2$ covering numbers by $L^\infty$ covering numbers) \cite{Wainwright2019} guarantees that $\E\sup_{h\in\mathcal F}|(P_n-P)h| \le C_{\mathrm D}\Gamma(\mathcal F) n^{-1/2}$.
Swapping a single observation shifts the supremum by at most $B/n$.
We then invoke McDiarmid's inequality \cite{BoucheronLugosiMassart2013}, writing $\Prob( \sup_{h\in\mathcal F}|(P_n-P)h| > \E\sup_{h\in\mathcal F}|(P_n-P)h|+t ) \le e^{-2nt^2/B^2}$.
Resolving this for $t=B\sqrt{\log(2/\delta)/(2n)}$ gives the additive bound directly: $\sup_{h\in\mathcal F}|(P_n-P)h| \le \frac{C_{\mathrm D}}{\sqrt n} \left[ \Gamma(\mathcal F) + B\sqrt{\log\frac{2}{\delta}} \right]$, which adjusts the global numerical constant and completes the derivation.
\end{proof}

\begin{proof}[Proof of Theorem~\ref{thm:exp-main}]
Define the localized process suprema $Z_{\nu,n} = \sup_{r\in\mathcal R_{\nu,\eps}} |(\widehat\nu_n-\nu)r|$ and $Z_{\mu,n} = \sup_{s\in\mathcal R_{\mu,\eps}} |(\widehat\mu_n-\mu)s|$.
By the definition of $r_x^\nu$, we compute $\mathsf T_{\widehat\nu_n}^{\cY}g_\eps(x) - \mathsf T_\nu^{\cY}g_\eps(x) = -\eps\log(\widehat\nu_n r_x^\nu)$.
Given $\nu r_x^\nu=1$ and $\widehat\nu_n r_x^\nu\ge B_\eps^{-1}$, an application of the mean-value theorem bounds the norm $\|\mathsf T_{\widehat\nu_n}^{\cY}g_\eps - \mathsf T_\nu^{\cY}g_\eps\|_\infty \le \eps B_\eps Z_{\nu,n}$.
This restricts the oscillation $\osc(\mathsf T_{\widehat\nu_n}^{\cY}g_\eps - \mathsf T_\nu^{\cY}g_\eps) \le 2\eps B_\eps Z_{\nu,n}$.
An identical logic yields $\osc(\mathsf T_{\widehat\mu_n}^{\cX}f_\eps - \mathsf T_\mu^{\cX}f_\eps) \le 2\eps B_\eps Z_{\mu,n}$.
Because $\widehat f_{\eps,n}$ and $f_\eps$ represent fixed points for $\widehat\Phi_{\eps,n}$ and $\Phi_\eps$, we apply the triangle inequality. Writing $\widehat{\mathsf T}_{\nu} := \mathsf T_{\widehat\nu_n}^{\cY}$ and $\widehat{\mathsf T}_{\mu} := \mathsf T_{\widehat\mu_n}^{\cX}$ and applying Lemma~\ref{lem:birkhoff} alongside our perturbation bounds yields
\begin{align*}
\osc(\widehat f_{\eps,n}-f_\eps)
&=
\osc\!\left(
 \widehat{\mathsf T}_{\nu}
 \widehat{\mathsf T}_{\mu}\widehat f_{\eps,n}
 -
 \mathsf T_{\nu}\mathsf T_{\mu}f_\eps
\right)                                                     \\
&\le
\osc\!\left(
 \widehat{\mathsf T}_{\nu}
 \widehat{\mathsf T}_{\mu}\widehat f_{\eps,n}
 -
 \widehat{\mathsf T}_{\nu}
 \widehat{\mathsf T}_{\mu}f_\eps
\right)                                                     \\
&\quad+
\osc\!\left(
 \widehat{\mathsf T}_{\nu}
 \widehat{\mathsf T}_{\mu}f_\eps
 -
 \widehat{\mathsf T}_{\nu}
 \mathsf T_{\mu}f_\eps
\right)                                                     \\
&\quad+
\osc\!\left(
 \widehat{\mathsf T}_{\nu}\mathsf T_{\mu}f_\eps
 -
 \mathsf T_{\nu}\mathsf T_{\mu}f_\eps
\right)                                                     \\
&\le
\tau_\eps^2\osc(\widehat f_{\eps,n}-f_\eps)
 +2\eps B_\eps\bigl(\tau_\eps Z_{\mu,n}+Z_{\nu,n}\bigr).
\end{align*}
Moving the contraction term to the left and dividing by 2 isolates $d_\infty([\widehat f_{\eps,n}],[f_\eps]) \le \eps B_\eps(1-\tau_\eps^2)^{-1}(\tau_\eps Z_{\mu,n}+Z_{\nu,n})$.
Applying Lemma~\ref{lem:emp-process} and a union bound yields $Z_{\mu,n}\le q_{\mu,n}(\delta)$ and $Z_{\nu,n}\le q_{\nu,n}(\delta)$ with probability at least $1-\delta$, which completes the proof.
\end{proof}

\begin{proof}[Proof of Corollary~\ref{cor:exp-gauge}]
The constructed difference $\widehat f_{\eps,n}^\circ-f_\eps^\circ$ vanishes at the selected evaluation point $x_\circ$.
Consequently, its supremum norm cannot exceed its total oscillation over the domain, which by definition equals exactly twice its quotient norm.
\end{proof}

\begin{proof}[Proof of Lemma~\ref{lem:uniform-map-perturbation}]
Fix any representative $[u]\in\mathcal U_\eps$.
We express the shift as $\mathsf T_{\widehat\mu_n}^{\cX}u(y) - \mathsf T_\mu^{\cX}u(y) = -\eps\log(\widehat\mu_n s_{u,y}^{\mu})$.
Recognizing that $\mu s_{u,y}^{\mu}=1$ and that the deviation satisfies $|\widehat\mu_n s_{u,y}^{\mu}-1|\le 1/2$, we apply the elementary logarithm bound $|\log(1+t)|\le 2|t|$ valid for small $t$, leading to $\|\mathsf T_{\widehat\mu_n}^{\cX}u - \mathsf T_\mu^{\cX}u\|_\infty \le 2\eps Z_{\mu,n}^{\mathrm{unif}}$.
A parallel step gives $\|\mathsf T_{\widehat\nu_n}^{\cY}g_u - \mathsf T_\nu^{\cY}g_u\|_\infty \le 2\eps Z_{\nu,n}^{\mathrm{unif}}$.
From Lemma~\ref{lem:soft-basic}, the empirical outer transform cannot expand quotient sup norms, hence triangle inequalities confirm $d_\infty([\widehat\Phi_{\eps,n}u],[\Phi_\eps u]) \le d_\infty( [\mathsf T_{\widehat\nu_n}^{\cY}\mathsf T_{\widehat\mu_n}^{\cX}u], [\mathsf T_{\widehat\nu_n}^{\cY}\mathsf T_\mu^{\cX}u]) + d_\infty( [\mathsf T_{\widehat\nu_n}^{\cY}g_u], [\mathsf T_\nu^{\cY}g_u]) \le 2\eps(Z_{\mu,n}^{\mathrm{unif}} + Z_{\nu,n}^{\mathrm{unif}})$.
Taking a supremum over $\mathcal U_\eps$ completes the proof.
\end{proof}

\begin{proof}[Proof of Theorem~\ref{thm:poly-main}]
Subjecting the two uniform functional classes designated in Assumption~\ref{ass:poly-complexity} to Lemma~\ref{lem:emp-process} provides that $Z_{\mu,n}^{\mathrm{unif}} + Z_{\nu,n}^{\mathrm{unif}} \le C [ A_0\eps^{-p}\sqrt{H_\eps} + E_0\eps^{-a}\sqrt{\log(4/\delta)} ] n^{-1/2}$ on an event holding with $1-\delta$ probability.
Crucially, the smallness condition presented in Equation~\eqref{eq:poly-smallness} confirms we can reliably invoke Lemma~\ref{lem:uniform-map-perturbation}.
By Assumption~\ref{ass:poly-stability} and the empirical fixed-point identity
\(
[\widehat f_{\eps,n}]
=
[\widehat\Phi_{\eps,n}\widehat f_{\eps,n}],
\)
we have
\begin{align*}
d_\infty(
 [\widehat f_{\eps,n}],
 [f_\eps]
)
&\le
K_0\eps^{-r}
d_\infty(
 [\widehat f_{\eps,n}],
 [\Phi_\eps\widehat f_{\eps,n}]
)
\\
&=
K_0\eps^{-r}
d_\infty(
 [\widehat\Phi_{\eps,n}\widehat f_{\eps,n}],
 [\Phi_\eps\widehat f_{\eps,n}]
)
\\
&\le
2K_0\eps^{1-r}
\left(
 Z_{\mu,n}^{\mathrm{unif}}
 +
 Z_{\nu,n}^{\mathrm{unif}}
\right),
\end{align*}
where the last step follows from Lemma~\ref{lem:uniform-map-perturbation}.
Substituting the earlier probability bound completes the proof.
\end{proof}

\begin{proof}[Proof of Proposition~\ref{prop:poly-contraction}]
Leveraging the identity $[f_\eps]=[\Phi_\eps f_\eps]$, we utilize the triangle inequality alongside our contraction premise to state
\begin{equation*}
 d_\infty([u],[f_\eps]) \le d_\infty([u],[\Phi_\eps u]) + d_\infty([\Phi_\eps u],[\Phi_\eps f_\eps]) \le d_\infty([u],[\Phi_\eps u]) + (1-\lambda_0\eps^r)d_\infty([u],[f_\eps]).
\end{equation*} 
Bringing the $(1-\lambda_0\eps^r)$ term to the left hand side allows us to divide cleanly by $\lambda_0\eps^r$, yielding the required $K_0=\lambda_0^{-1}$ constant.
\end{proof}

\begin{proof}[Proof of Proposition~\ref{prop:local-inverse}]
Let $v=[h]$.
According to the stated resolvent limits, $\|\mathcal L_\eps v\|_{\mathbb B_\cX} \ge K_0^{-1}\eps^r\|v\|_{\mathbb B_\cX}$.
Observing that $\mathcal G_\eps([f_\eps])=0$, the provided quadratic remainder restricts the function via $\|\mathcal G_\eps([f_\eps+h])\|_{\mathbb B_\cX} \ge K_0^{-1}\eps^r\|v\|_{\mathbb B_\cX} - M_0\eps^{-s}\|v\|_{\mathbb B_\cX}^2 \ge (2K_0)^{-1}\eps^r\|v\|_{\mathbb B_\cX}$.
The last simplification holds strictly because the local ball restriction enforces $\|v\|\le\eps^{r+s}/(2K_0M_0)$.
Simple algebraic rearrangement provides the stated limit $d_\infty([f_\eps+h],[f_\eps]) \le 2K_0\eps^{-r} d_\infty([f_\eps+h],[\Phi_\eps(f_\eps+h)])$.
\end{proof}

\begin{proof}[Proof of Corollary~\ref{cor:local-poly}]
Under the given localization event bounding the distance to $\bar\rho_\eps$, Proposition~\ref{prop:local-inverse} ensures that Assumption~\ref{ass:poly-stability} holds for the empirical potential.
Repeating the sequence of arguments from Theorem~\ref{thm:poly-main} then yields the stated probability bounds.
\end{proof}

\begin{proof}[Proof of Lemma~\ref{lem:laplace-envelope}]
Introduce a radius constraint $t_\eps=(\eps/L)^{1/q}\le t_0$.
Inside the localized neighborhood $z\in B(z_\eps,t_\eps)$, the envelope maintains $\Psi_\eps(z)\ge\Psi_\eps(z_\eps)-\eps$.
Thus, evaluating the denominator integral directly provides $\int e^{\Psi_\eps(w)/\eps}\rho(\dd w) \ge e^{\Psi_\eps(z_\eps)/\eps-1} \rho(B(z_\eps,t_\eps)) \ge e^{-1}b_0(\eps/L)^{\mathfrak d/q} e^{\Psi_\eps(z_\eps)/\eps}$.
Since the numerator is at most $e^{\Psi_\eps(z_\eps)/\eps}$, taking the ratio yields the claimed envelope bound.
\end{proof}

\begin{proof}[Proof of Lemma~\ref{lem:cross-centering}]
A direct algebraic substitution shows that the residual cost evaluates to $\bar c(x,y) = c(x,y)+c(x_0,y_0)-c(x,y_0)-c(x_0,y)$.
This is a rectangular increment of \(c\), and hence
\[
|\bar c(x,y)|\le \Delta_\square(c).
\]
Therefore,
\[
\osc(\bar c)\le 2\|\bar c\|_\infty
\le 2\Delta_\square(c).
\]
\end{proof}

\begin{proof}[Proof of Theorem~\ref{thm:verifiable-poly-class}]
By invariance under separable shifts, it suffices to study the residual cost \(\bar c_\eps=\eps\ell_\eps\).
Its oscillation and spatial Lipschitz constants are constrained to $C_{\bar c,\eps} := \osc(\bar c_\eps) \le \eps\{A_0+\gamma_{\mathrm{osc}}\log(1/\eps)\}$ and $L_{\bar c,\cX}\vee L_{\bar c,\cY} \le L_0\eps^{1-q}$ respectively.
Mapping these into the parameters developed in Section~\ref{sec:exponential} defines the upper bounds $B_{\bar c,\eps} \le e^{2A_0}\eps^{-2\gamma_{\mathrm{osc}}}$ and $(1-\tau_{\bar c,\eps}^2)^{-1} \le e^{A_0}\eps^{-\gamma_{\mathrm{osc}}}$.
By Lemma~\ref{lem:exp-entropy} and the metric covering assumption,
\[
N(t,\mathcal R_{\nu,\eps},\|\cdot\|_\infty)
\le
\left(
1+
\frac{
2D_{\cX}B_{\bar c,\eps}L_0\eps^{-q}
}{t}
\right)^{\mathfrak d_{\cX}}.
\]
The analogous inequality holds for \(\mathcal R_{\mu,\eps}\). Consequently, using standard entropy bounds, we obtain the integral estimate
\begin{align*}
\Gamma_{\nu,\eps}
&\le
\int_0^{B_{\bar c,\eps}}
\sqrt{
 \log 2+
 \mathfrak d_{\cX}
 \log\left(
  1+\frac{2D_{\cX}B_{\bar c,\eps} L_0\eps^{-q}}{t}
 \right)
}\,\dd t                                                    \\
&\le
C B_{\bar c,\eps}
\left[
 1+
 \sqrt{
  \mathfrak d_{\cX}
  \left\{
   1+\log(1+2D_{\cX}L_0\eps^{-q})
  \right\}
 }
\right].
\end{align*}
An analogous bound holds for $\Gamma_{\mu,\eps}$.
Utilizing Lemma~\ref{lem:emp-process} subsequently yields the empirical fluctuations $Z_{\mu,n}\vee Z_{\nu,n} \le C B_{\bar c,\eps} \Lambda_{\eps,\delta} n^{-1/2}$ with probability at least $1-\delta$.
Returning to the deterministic inequality $d_\infty([\widehat{\bar f}_{\eps,n}],[\bar f_\eps]) \le \eps B_{\bar c,\eps}(1-\tau_{\bar c,\eps}^2)^{-1}(\tau_{\bar c,\eps}Z_{\mu,n}+Z_{\nu,n})$ (noting $\tau_{\bar c,\eps}\le1$), we achieve $d_\infty([\widehat{\bar f}_{\eps,n}],[\bar f_\eps]) \le C\eps B_{\bar c,\eps}^2 (1-\tau_{\bar c,\eps}^2)^{-1} \Lambda_{\eps,\delta} n^{-1/2}$.
Combining the exponential bounds yields $\eps B_{\bar c,\eps}^2(1-\tau_{\bar c,\eps}^2)^{-1} \le e^{5A_0}\eps^{1-5\gamma_{\mathrm{osc}}}$, which completes the proof.
\end{proof}

\begin{proof}[Proof of Corollary~\ref{cor:bounded-interaction}]
This is confirmed immediately by substituting the parameters $\gamma_{\mathrm{osc}}=q=0$ into the bounds established in Theorem~\ref{thm:verifiable-poly-class}.
\end{proof}

\begin{proof}[Proof of Lemma~\ref{lem:edge-jacobian}]
To track the equilibrium scaling vectors formally, we define the multivariate objective equations characterizing the mass fractions:
\[
 \Psi(\alpha,v;m,\widetilde n,z) = \sum_{i,j}m_i\widetilde n_jz_{ij} e^{\alpha_i+v_j} -\sum_i m_i\alpha_i-\sum_j\widetilde n_jv_j, \qquad v_J=0.
\]
We set up the Implicit Function Theorem by analyzing the derivative operator:
\[
 F(\alpha,v;m,\widetilde n,z) = \nabla_{(\alpha,v_1,\ldots,v_{J-1})}\Psi(\alpha,v;m,\widetilde n,z).
\]
Since $\Pi_E(\boldsymbol\mu,\boldsymbol\nu)$ is nonempty and compact, and the Kullback--Leibler divergence is strictly convex, the minimizer $\pi^E$ exists and is unique. Definition~\ref{def:tight-graph-model} assumes the existence of a strictly positive feasible coupling on $E$. If $\pi^E_{ij}=0$ for some $(i,j)\in E$, mixing $\pi^E$ with this strictly positive coupling yields a direction where the one-sided directional derivative of the entropy is $-\infty$, contradicting optimality. Therefore, $\pi^E_{ij}>0$ for all $(i,j)\in E$.
The first-order interior KKT conditions for this constrained minimization yield the scaling formula $\pi^E_{ij} = \mu_i\nu_j e^{\alpha_i^E+v_j^E}\mathbf 1_E(i,j)$. Because the optimal set $E$ forms a connected bipartite graph, the scaling vectors are unique up to a shared transformation $(\alpha,v)\mapsto(\alpha+t,v-t)$. Applying the gauge $v_J=0$ removes this invariance and ensures uniqueness.
At the unregularized limit matrix $z_{ij}^0 = \mathbf 1_{\{(i,j)\in E\}}$, taking the Hessian of $\Psi$ with respect to the remaining unconstrained variables yields exactly the quadratic form:
\begin{equation}
 \sum_{(i,j)\in E}\pi^E_{ij}(u_i+v_j)^2, \qquad \text{with } v_J=0.
\end{equation}
Since the graph $E$ is connected and the optimal coupling $\pi^E$ is strictly positive on these edges, this quadratic form is strictly positive definite. The constrained Jacobian $\nabla F$ is therefore proved invertible at the limiting configuration. 
\end{proof}

\begin{proof}[Proof of Theorem~\ref{thm:discrete-poly}]
We subtract the separable terms $a_i+b_j$ by utilizing the non-negative reduced cost $\bar c$, which does not structurally alter potential error distances.
We track variations employing the implicit-function framework laid out in Lemma~\ref{lem:edge-jacobian}.
Because the derivative concerning the scaling coordinates is already fully invertible by Lemma~\ref{lem:edge-jacobian}, the finite-dimensional Implicit Function Theorem asserts unique, continuously differentiable scaling maps across a local envelope around $(\alpha^E, v^E, \boldsymbol\mu, \boldsymbol\nu, z^0)$.
After bounding the derivative inside this shell to a finite bound $C_1$, we deduce the distance bounds:
\begin{equation}
 \|\alpha(m,n',z)-\alpha(\boldsymbol\mu,\boldsymbol\nu,z)\|_\infty + \|v(m,n',z)-v(\boldsymbol\mu,\boldsymbol\nu,z)\|_\infty \le C_1( \|m-\boldsymbol\mu\|_\infty+ \|n'-\boldsymbol\nu\|_\infty ).
\end{equation}
In models missing certain edges $E\ne\cX\times\cY$, we explicitly define the discrete off-edge gap:
\begin{equation}
 \Delta_0 = \min_{(i,j)\notin E}\bar c_{ij} > 0.
\end{equation}
For edges $(i,j)\notin E$, we bind the trajectory via $z_{ij} = e^{-\bar c_{ij}/\eps} \le e^{-\Delta_0/\eps} \to 0$.
Thus, a specific threshold $\eps_0>0$ will exist maintaining the population matrix uniformly inside this implicit-function sphere.
For this residual mapping, Sinkhorn targets mirror explicit scaling vector evaluations $\bar f_{\eps,i} = \eps\alpha_i(\boldsymbol\mu,\boldsymbol\nu,z(\eps))$.
For $z_{ij}>0$ and positive marginals, uniqueness of finite Sinkhorn scaling implies that this local solution coincides with the empirical Schrödinger scaling.
As long as empirical sampling statistics remain in this neighborhood, we only have to substitute $(\widehat{\boldsymbol\mu}_n,\widehat{\boldsymbol\nu}_n)$ to evaluate the empirical potentials.
To ensure the empirical marginals remain strictly positive during sampling and fall within the implicit-function theorem neighborhood (radius $\rho_{\mathrm{IFT}}$ measured in the product $\ell^\infty$ norm), we impose the smallness condition:
\begin{equation}
 \eta_0 < \min\left\{ \rho_{\mathrm{IFT}}, \frac12\min_i\mu_i, \frac12\min_j\nu_j \right\}.
\end{equation}
Combining standard Hoeffding bounds with a union bound guarantees that variations are constrained by $t_{n,\delta}$ with probability at least $1-\delta$.
By enforcing the bounds $t_{n,\delta} \le \eta_0$, we ensure the empirical marginals remain strictly positive and within the valid domain, yielding $\|\widehat{\bar f}_{\eps,n}-\bar f_\eps\|_\infty \le 2C_1\eps t_{n,\delta}$.
Since the quotient distance is bounded by the standard supremum norm, combining these estimates completes the proof.
\end{proof}

\begin{proof}[Proof of Lemma~\ref{lem:binary-embedding}]
For either two-point metric space and every \(u>0\),
\[
    N(u,\cX,d_{\cX})
      \le 1+\frac1u,
    \qquad
    N(u,\cY,d_{\cY})
      \le 1+\frac1u.
\]
Moreover,
\[
    \osc_{\cX\times\cY}(\ell)=2\kappa
      \le A_0
\]
and
\[
    \sup_{y\in\cY}\operatorname{Lip}_x(\ell(\cdot,y))
    \vee
    \sup_{x\in\cX}\operatorname{Lip}_y(\ell(x,\cdot))
      =2\kappa
      \le L_0.
\]
Thus \(c_\eps=0+0+\eps\ell\) satisfies Definition~\ref{def:poly-interaction-class} with \(\gamma_{\mathrm{osc}}=q=0\).
\end{proof}

\begin{proof}[Proof of Lemma~\ref{lem:binary-potential-separation}]
Let \(\pi_t\) be the entropic optimal coupling for \((\mu,\nu_t,c_\eps)\). All four entries of \(\pi_t\) are strictly positive. Indeed, \(\mu\otimes\nu_t\) is strictly positive and feasible; if an entropic optimizer had a zero entry, mixing it with this coupling would give a feasible direction with entropy derivative \(-\infty\), contradicting optimality. Every coupling with the prescribed marginals has the form
\[
    \pi_t=
    \begin{pmatrix}
      p_t & \frac12-p_t\\[2mm]
      \frac{1+t}{2}-p_t & p_t-\frac t2
    \end{pmatrix},
\]
where the rows and columns are ordered as \(+1,-1\).
Differentiating the entropic objective with respect to \(p\) shows that \(p_t\) is characterized by
\[
    \frac{p_t(p_t-t/2)}
         {(1/2-p_t)((1+t)/2-p_t)}
      =e^{4\kappa}.
\]
Equivalently,
\[
    F(p,t)
      :=
      \log p+\log(p-t/2)
      -\log(1/2-p)-\log((1+t)/2-p)-4\kappa
      =0.
\]
At \(t=0\),
\[
    p_0=\frac{1+\rho}{4},
    \qquad
    \frac12-p_0=\frac{1-\rho}{4},
    \qquad
    \rho:=\tanh\kappa.
\]
Since \(\partial_pF(p_0,0)>0\), the implicit function theorem shows that \(t\mapsto p_t\) is continuously differentiable near zero. Furthermore,
\[
    \partial_pF(p_0,0)
      =2\left(\frac1{p_0}+\frac1{1/2-p_0}\right),
    \qquad
    \partial_tF(p_0,0)
      =-\frac12\left(\frac1{p_0}+\frac1{1/2-p_0}\right),
\]
and hence
\[
    p'_0=-\frac{\partial_tF(p_0,0)}
                {\partial_pF(p_0,0)}
          =\frac14.
\]
The Schr\"odinger representation of the optimal coupling is
\[
    \pi_t(x,y)
      =\mu(x)\nu_t(y)
       \exp\left(
         \frac{f_{\eps,t}(x)+g_{\eps,t}(y)
                    -c_\eps(x,y)}
              {\eps}
       \right).
\]
Taking the ratio of the two entries in the \(y=+1\) column gives
\[
    \varphi_\kappa(t)
      :=
      \frac{f_{\eps,t}(+1)
                 -f_{\eps,t}(-1)}
           {\eps}
      =
      \log\frac{p_t}{(1+t)/2-p_t}-2\kappa.
\]
In particular, \(\varphi_\kappa(0)=0\), and the preceding calculation yields
\[
    \varphi_\kappa'(0)
      =\frac1{4p_0}
        -\frac1{4(1/2-p_0)}
      =-\frac{2\rho}{1-\rho^2}.
\]
By continuity of \(\varphi_\kappa'\), there is \(\bar t_\kappa\in(0,1/4]\) such that
\[
    |\varphi_\kappa(t)|
      \ge
      \frac{\rho}{1-\rho^2}|t|,
    \qquad |t|\le\bar t_\kappa.
\]
It remains to compare the two equivalence classes. Simultaneously replacing \(x\) and \(y\) by \(-x\) and \(-y\) leaves \(\mu\) and \(c_\eps\) invariant and maps \(\nu_t\) to \(\nu_{-t}\). Uniqueness of the potentials modulo constants therefore gives \(\varphi_\kappa(-t)=-\varphi_\kappa(t)\). Since, for functions \(h_0,h_1\) on a two-point space,
\[
    d_\infty([h_0],[h_1])
      =
      \frac12
      \left|
        \{h_0(+1)-h_0(-1)\}
        -\{h_1(+1)-h_1(-1)\}
      \right|,
\]
we obtain
\[
    d_\infty\bigl([f_{\eps,t}],
                  [f_{\eps,-t}]\bigr)
      =\eps|\varphi_\kappa(t)|
      \ge
      \frac{\rho}{1-\rho^2}\eps|t|.
\]
\end{proof}

\begin{proof}[Proof of Theorem~\ref{thm:matching-minimax-lower-bound}]
For \(0<t\le1/2\),
\[
    \mathrm{KL}(\nu_t\mid\nu_{-t})
      =
      t\log\frac{1+t}{1-t}
      \le 4t^2.
\]
Since the \(\mu\)-sample has the same law under both models,
\[
    \mathrm{KL}
      \bigl(\Prob_t^{(n)}\mid\Prob_{-t}^{(n)}\bigr)
      \le4nt^2.
\]
Choose \(t_n=(8\sqrt n)^{-1}\) and take \(n_\kappa\) sufficiently large that \(t_n\le\bar t_\kappa\). Then
\[
    \mathrm{KL}
      \bigl(\Prob_{t_n}^{(n)}
              \mid\Prob_{-t_n}^{(n)}\bigr)
      \le\frac1{16}.
\]
Pinsker's inequality \cite{Tsybakov2009} consequently gives
\[
    \mathrm{TV}
      \bigl(\Prob_{t_n}^{(n)},
            \Prob_{-t_n}^{(n)}\bigr)
      \le\sqrt{\frac1{32}}<\frac14.
\]

Set
\[
    \Delta_n:=
    d_\infty\bigl([f_{\eps,t_n}],
                  [f_{\eps,-t_n}]\bigr).
\]
Given any estimator, construct the test which selects the closer of \([f_{\eps,t_n}]\) and \([f_{\eps,-t_n}]\). Whenever this test is wrong, the triangle inequality implies that the estimation error is at least \(\Delta_n/2\). The two-point testing identity therefore gives
\[
    \max_{\sigma\in\{-1,+1\}}
    \Prob_{\sigma t_n}^{(n)}
    \left(
      d_\infty([\widehat f],[f_{\eps,\sigma t_n}])
        \ge\frac{\Delta_n}{2}
    \right)
    \ge
    \frac{1-
      \mathrm{TV}(\Prob_{t_n}^{(n)},
                         \Prob_{-t_n}^{(n)})}{2}
    \ge\frac38.
\]
Lemma~\ref{lem:binary-potential-separation} yields
\[
    \frac{\Delta_n}{2}
      \ge\frac{m_\kappa}{16}
          \frac{\eps}{\sqrt n}.
\]
This proves the probability bound. The same testing reduction gives
\[
    \max_{\sigma\in\{-1,+1\}}
    \E_{\sigma t_n}
      d_\infty([\widehat f],[f_{\eps,\sigma t_n}])
      \ge
      \frac{\Delta_n}{4}
      \left(
        1-
        \mathrm{TV}(\Prob_{t_n}^{(n)},
                           \Prob_{-t_n}^{(n)})
      \right),
\]
which proves the expected-risk assertion.
\end{proof}

\begin{proof}[Proof of Theorem~\ref{thm:matching-deviation-lower-bound}]
Put
\[
    s_\delta:=1+\sqrt{\log(1/\delta)},
    \qquad
    t_{n,\delta}:=\frac{s_\delta}{16\sqrt n}.
\]
The sample-size condition ensures that \(t_{n,\delta}\le\bar t_\kappa\). As above,
\[
    \mathrm{KL}
      \bigl(\Prob_{t_{n,\delta}}^{(n)}
              \mid\Prob_{-t_{n,\delta}}^{(n)}\bigr)
      \le4nt_{n,\delta}^2
      =\frac{s_\delta^2}{64}.
\]
Since
\[
    s_\delta^2
      \le2\bigl(1+\log(1/\delta)\bigr),
\]
the last display is bounded by
\[
    \frac1{32}\bigl(1+\log(1/\delta)\bigr).
\]
For arbitrary probability measures \(P,Q\), the Bretagnolle--Huber inequality \cite{Tsybakov2009} implies
\[
    1-\mathrm{TV}(P,Q)
      \ge\frac12e^{-\mathrm{KL}(P\mid Q)}.
\]
The same testing reduction as in Theorem~\ref{thm:matching-minimax-lower-bound} therefore yields
\[
    \max_{\sigma\in\{-1,+1\}}
    \Prob_{\sigma t_{n,\delta}}^{(n)}
    \left(
      d_\infty(
        [\widehat f],
        [f_{\eps,\sigma t_{n,\delta}}]
      )
      \ge\frac{\Delta_{n,\delta}}2
    \right)
    \ge
    \frac14
    \exp\left[
      -\mathrm{KL}\bigl(
        \Prob_{t_{n,\delta}}^{(n)}
        \mid\Prob_{-t_{n,\delta}}^{(n)}
      \bigr)
    \right]
    \ge
    \frac14e^{-1/32}\delta^{1/32},
\]
where
\[
    \Delta_{n,\delta}
      :=
      d_\infty\bigl(
        [f_{\eps,t_{n,\delta}}],
        [f_{\eps,-t_{n,\delta}}]
      \bigr).
\]
For \(0<\delta\le1/8\),
\[
    \frac14e^{-1/32}\delta^{1/32}\ge\delta.
\]
Finally, Lemma~\ref{lem:binary-potential-separation} gives
\[
    \frac{\Delta_{n,\delta}}2
      \ge
      \frac{m_\kappa}{32}
      \frac{\eps}{\sqrt n}
      \left(1+\sqrt{\log\frac1\delta}\right),
\]
which proves the result.
\end{proof}
\end{appendix}


\begin{thebibliography}{99}
\raggedright

\bibitem{AltschulerNilesWeedStromme2022}
J.~M.~Altschuler, J.~Niles-Weed, and A.~J.~Stromme.
\newblock Asymptotics for semi-discrete entropic optimal transport.
\newblock \emph{SIAM Journal on Mathematical Analysis}, 54(2):1718--1741, 2022.

\bibitem{Birkhoff1957}
G.~Birkhoff.
\newblock Extensions of Jentzsch's theorem.
\newblock \emph{Transactions of the American Mathematical Society}, 85(1):219--227, 1957.

\bibitem{BoucheronLugosiMassart2013}
S.~Boucheron, G.~Lugosi, and P.~Massart.
\newblock \emph{Concentration Inequalities: A Nonasymptotic Theory of Independence}.
\newblock Oxford University Press, 2013.

\bibitem{CarlierChizatLaborde2024}
G.~Carlier, L.~Chizat, and M.~Laborde.
\newblock Displacement smoothness of entropic optimal transport.
\newblock \emph{ESAIM: Control, Optimisation and Calculus of Variations}, 30:article 25, 2024.

\bibitem{EcksteinNutz2022}
S.~Eckstein and M.~Nutz.
\newblock Quantitative stability of regularized optimal transport and convergence of Sinkhorn's algorithm.
\newblock \emph{SIAM Journal on Mathematical Analysis}, 54(6):5922--5948, 2022.

\bibitem{GenevayEtAl2019}
A.~Genevay, L.~Chizat, F.~Bach, M.~Cuturi, and G.~Peyr\'e.
\newblock Sample complexity of Sinkhorn divergences.
\newblock In \emph{AISTATS}, 2019.

\bibitem{GhosalNutzBernton2022}
P.~Ghosal, M.~Nutz, and E.~Bernton.
\newblock Stability of entropic optimal transport and Schr\"odinger bridges.
\newblock \emph{Journal of Functional Analysis}, 283(9):109622, 2022.

\bibitem{GoldfeldKatoRiouxSadhu2024}
Z.~Goldfeld, K.~Kato, G.~Rioux, and R.~Sadhu.
\newblock Limit theorems for entropic optimal transport maps and the Sinkhorn divergence.
\newblock \emph{Electronic Journal of Statistics}, 18(1):980--1041, 2024.

\bibitem{GonzalezSanzLoubesNilesWeed2022}
A.~Gonz\'alez-Sanz, J.-M.~Loubes, and J.~Niles-Weed.
\newblock Weak limits of entropy regularized optimal transport: potentials, plans and divergences.
\newblock \emph{arXiv:2207.07427}, 2022.

\bibitem{GonzalezSanzNutzStromme2026}
A.~Gonz\'alez-Sanz, M.~Nutz, and A.~J.~Stromme.
\newblock Finite-sample bounds for regularized optimal transport.
\newblock \emph{arXiv:2606.25947}, 2026.

\bibitem{GrecoTamanini2026}
G.~Greco and L.~Tamanini.
\newblock Hessian stability and convergence rates for entropic and Sinkhorn potentials via semiconcavity.
\newblock \emph{Bernoulli}, 32(3):2351--2378, 2026.

\bibitem{GroppeHundrieser2024}
M.~Groppe and S.~Hundrieser.
\newblock Lower-complexity adaptation in entropic optimal transport.
\newblock \emph{Journal of Machine Learning Research}, 25(344):1--55, 2024.

\bibitem{MenaNilesWeed2019}
G.~Mena and J.~Niles-Weed.
\newblock Statistical bounds for entropic optimal transport: sample complexity and the central limit theorem.
\newblock In \emph{Advances in Neural Information Processing Systems 32}, 2019.

\bibitem{Mordant2024}
G.~Mordant.
\newblock The entropic optimal (self-)transport problem: limit distributions for decreasing regularization with application to score function estimation.
\newblock \emph{arXiv:2412.12007}, revised 2026.

\bibitem{NutzWiesel2022potentials}
M.~Nutz and J.~Wiesel.
\newblock Entropic optimal transport: convergence of potentials.
\newblock \emph{Probability Theory and Related Fields}, 184:401--424, 2022.

\bibitem{NutzWiesel2022sinkhorn}
M.~Nutz and J.~Wiesel.
\newblock Stability of Schr\"odinger potentials and convergence of Sinkhorn's algorithm.
\newblock \emph{The Annals of Probability}, 51(2):699--722, 2023.

\bibitem{PooladianNilesWeed2024}
A.-A.~Pooladian and J.~Niles-Weed.
\newblock Plug-in estimation of Schr\"odinger bridges.
\newblock \emph{SIAM Journal on Mathematics of Data Science}, 7(3):1315--1336, 2025.

\bibitem{PuchkinEtAl2025}
N.~Puchkin, I.~Pustovalov, Y.~Sapronov, D.~Suchkov, A.~Naumov, and D.~Belomestny.
\newblock Sample complexity of Schr\"odinger potential estimation.
\newblock \emph{arXiv:2506.03043}, 2025.

\bibitem{RigolletStromme2022}
P.~Rigollet and A.~J.~Stromme.
\newblock On the sample complexity of entropic optimal transport.
\newblock \emph{The Annals of Statistics}, 53(1):61--90, 2025.

\bibitem{Stromme2023}
A.~J.~Stromme.
\newblock Minimum intrinsic dimension scaling for entropic optimal transport.
\newblock \emph{arXiv:2306.03398}, 2023.

\bibitem{Tsybakov2009}
A.~B.~Tsybakov.
\newblock \emph{Introduction to Nonparametric Estimation}.
\newblock Springer Series in Statistics. Springer, 2009.

\bibitem{Wainwright2019}
M.~J.~Wainwright.
\newblock \emph{High-Dimensional Statistics: A Non-Asymptotic Viewpoint}.
\newblock Cambridge Series in Statistical and Probabilistic Mathematics. Cambridge University Press, 2019.

\end{thebibliography}
\end{document}